\PassOptionsToPackage{table}{xcolor}
\documentclass[sigconf]{acmart}

\AtBeginDocument{}

\copyrightyear{2026}
\acmYear{2026}
\setcopyright{cc}
\setcctype{by}
\acmConference[CIKM '26]{Proceedings of the 35th ACM International Conference on Information and Knowledge Management}{November 07--11, 2026}{Rome, Italy}
\acmBooktitle{Proceedings of the 35th ACM International Conference on Information and Knowledge Management (CIKM '26), November 07--11, 2026, Rome, Italy}
\acmDOI{10.1145/3799682.3841143}
\acmISBN{979-8-4007-2539-5/2026/11}

\DeclareUnicodeCharacter{2217}{\textasteriskcentered}

\usepackage{amsthm}
\usepackage{textcomp}
\usepackage{multirow}
\usepackage{algorithm}
\usepackage[noend]{algorithmic}
\usepackage{balance}

\newtheorem{proposition}{Proposition}
\newtheorem{theorem}{Theorem}

\newtheorem{remark}{Remark}
\newtheorem{definition}{Definition}

\newcommand{\R}{\mathbb{R}}
\newcommand{\E}{\mathbb{E}}

\newcommand{\kcm}{\textsc{kcm}}

\newcommand{\kar}{\textsc{kar}}

\begin{document}

\title{Mitigating Convergence Collapse in Fixed-Target Anomaly Detectors via Kernel-Anchored Locality Regularization}

\newcommand{\affilCorresponding}{Universit\'e Paris-Saclay, CNRS, CentraleSup\'elec, LISN UMR 9015}
\newcommand{\affilCoauthor}{Universit\'e Paris-Saclay, CNRS, CentraleSup\'elec, Laboratoire Interdisciplinaire des Sciences du Num\'erique}

\author{Jos\'e Lucas De Melo Costa}
\authornote{Corresponding author.}
\orcid{0009-0007-5372-9101}
\affiliation{%
  \institution{\affilCorresponding}
  \city{Orsay}
  \country{France}}
\email{jose.lucas-de-melo-costa@universite-paris-saclay.fr}

\author{Fabrice Popineau}
\orcid{0000-0002-2941-9046}
\affiliation{%
  \institution{\affilCoauthor}
  \city{Orsay}
  \country{France}}
\email{fabrice.popineau@centralesupelec.fr}

\author{Arpad Rimmel}
\orcid{0009-0009-7028-3644}
\affiliation{%
  \institution{\affilCoauthor}
  \city{Orsay}
  \country{France}}
\email{arpad.rimmel@centralesupelec.fr}

\author{Bich-Li\^en Doan}
\orcid{0009-0001-5853-2486}
\affiliation{%
  \institution{\affilCoauthor}
  \city{Orsay}
  \country{France}}
\email{Bich-Lien.Doan@centralesupelec.fr}

\renewcommand{\shortauthors}{José Lucas De Melo Costa∗, Fabrice Popineau, Arpad Rimmel, \& Bich-Liên Doan}

\begin{abstract}
A family of tabular anomaly detectors trains a neural map toward a fixed target under squared-error loss and scores anomalies by the test-time residual; contraction matching, one-step rectified flow, and reconstruction autoencoders all fit this template. We characterize a \emph{convergence collapse}: better optimization makes the detector worse. At convergence, the learned map tracks the target even off-distribution, so the residual signal vanishes on anomalies as well as on normal data. These detectors therefore rely on implicit non-convergence (early stopping, capacity caps) to retain signal. We argue this is structural: effective anomaly detection requires a \emph{locality constraint} that blocks unconstrained extrapolation. Classical detectors ($k$NN, KDE, isolation forests, LOF) enforce locality explicitly; fixed-target neural detectors do not. We formalize the connection by showing that the kernel-regression analog of a fixed-target detector is a finite-bandwidth Nadaraya--Watson smoother, which we call \emph{Kernel Contraction Matching} (KCM). KCM is closed-form, training-free, and CPU-efficient, yet matches established neural baselines on ADBench. Building on this bridge, we introduce the \emph{Kernel-Anchored Regularizer} (KAR), which penalizes deviation of the neural prediction from a kernel-weighted average of training targets. Across collapse-prone ADBench datasets and three backbones, KAR mitigates collapse and improves AUROC under prolonged training.
\end{abstract}



\maketitle

\section{Introduction}
\label{sec:intro}

\begin{figure*}[t]
  \centering
  \includegraphics[width=0.88\linewidth]{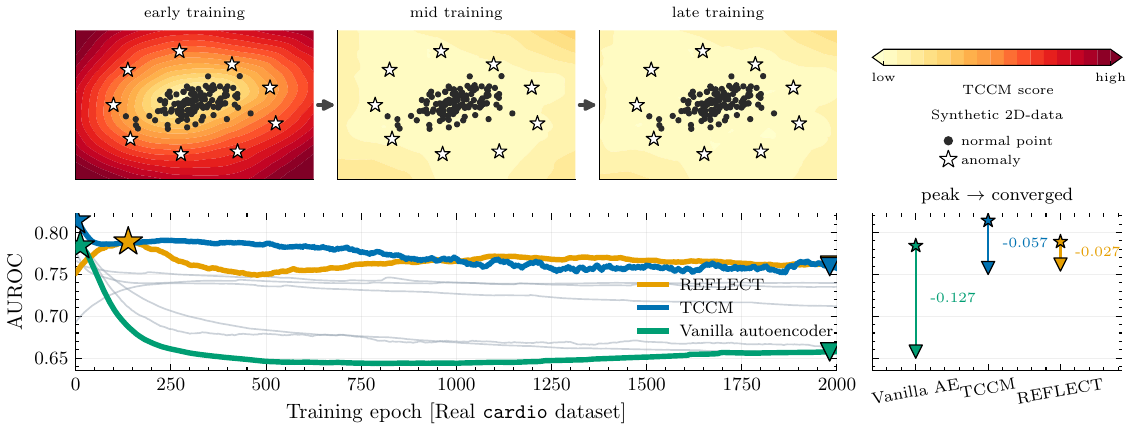}
  \caption{\textbf{Convergence collapse across fixed-target anomaly detectors.}
  \emph{Top:} the anomaly-score field flattens along training, until the
  anomalies (stars) are no longer separated from the normal cluster (dots).
  \emph{Bottom left:} on ADBench~\cite{hanADBenchAnomalyDetection2022}
  \texttt{7\_Cardiotocography}, the methods reach an early peak
  ($\star$) and then decay toward convergence ($\triangledown$).
  \emph{Bottom right:} the peak-to-converged AUROC drops are $0.128$, $0.051$,
  and $0.017$.}
  \Description{A four-panel figure showing that fixed-target anomaly detectors
  improve early in training and then collapse as training continues; the anomaly
  scores flatten, and the AUROC drops while the training loss keeps decreasing.}
  \label{fig:collapse}
\end{figure*}

Tabular anomaly detection identifies records that depart from normal behaviour, a long-standing problem with stakes in fraud screening, intrusion detection, condition monitoring, and medical diagnosis~\cite{hanADBenchAnomalyDetection2022, boumanUnsupervisedAnomalyDetection2024}. A growing number of detectors share one design, the \emph{fixed-target} template: a neural network is trained on normal data to map every input to a predetermined \emph{target}, a value set in advance as a function of the input alone, and a test point is scored by its \emph{residual}, the distance between the network's output and that target. The reconstruction autoencoder uses the input itself as its target~\cite{sattarovDiffusionScheduledDenoisingAutoencoders2025, ruffUnifyingReviewDeep2021}; flow-matching detectors instead regress a velocity field~\cite{liHowWhyTaming2025, liScalableExplainableProvably2025}, with contraction matching aiming that velocity at a fixed point. Fitting a fixed target is ordinary squared-error regression, with no adversarial game, no density model, and no anomaly labels; the score is the residual, read off in one forward pass.

This approach carries an implicit promise. Trained by regression, the detector should only get sharper with more epochs and a lower training loss, so training to convergence looks like good practice. \emph{We find the opposite.} As the squared-error loss is minimized, detection quality first improves and then degrades. The failure is intrinsic to this class of methods, and we call it \emph{convergence collapse}. At the optimum the network reproduces the fixed target on the normal data; because the target is a fixed function and the network is free to extrapolate, it reproduces the target off the normal distribution too, exactly where anomalies lie. The residual then shrinks on anomalies just as it shrinks on normal points, and AUROC declines even as the loss keeps falling. The collapsed detector is not mistrained: it has correctly solved its squared-error problem, and that problem was not detection. The pattern is not specific to one backbone: it has been studied for reconstruction autoencoders, where an expressive network reconstructs anomalies as accurately as normal points~\cite{gongMemorizingNormalityDetect2019, ghafourianTargetedCollapseRegularized2025}, and we show the same failure across contraction matching and flow matching. Figure~\ref{fig:collapse} shows an autoencoder, a contraction matching detector (TCCM \cite{liScalableExplainableProvably2025}), and a flow-matching detector (REFLECT \cite{beizaeeREFLECTRectifiedFlows2025a}), each climbing to an early peak and then decaying toward convergence. The published strength of fixed-target detectors therefore rests on \emph{implicit non-convergence}: early stopping and capacity caps that are tuned, rarely reported, and increasingly automated without naming the failure they prevent~\cite{huangEntropyStopUnsupervisedDeep2024}.

The root cause is a missing constraint. The squared-error loss constrains the network only on the training data, so its score at a test point can be set by free extrapolation rather than by the nearby training points. The methods that dominate tabular benchmarks~\cite{hanADBenchAnomalyDetection2022, boumanUnsupervisedAnomalyDetection2024}, among them $k$-nearest-neighbour distance, the local outlier factor~\cite{breunigLOFIdentifyingDensitybased2000}, kernel density estimation, one-class SVM~\cite{scholkopfSupportVectorMethod1999}, and isolation forest~\cite{liuIsolationForest2008}, rest on the opposite discipline: the score at a point is set mostly by training points in a \emph{bounded neighbourhood} of it. We call this property \emph{locality}. Classical detectors enforce it by construction; a fixed-target neural detector, left to converge, does not. Convergence collapse is the absence of a locality constraint.

Locality also points to a cure.\footnote{Code, configurations, and run logs: \url{https://github.com/jose-melo/kernel-contraction-matching}.} Consider the kernel-regression analog of a fixed-target detector, the smoother a wide network approaches in its lazy training regime~\cite{jacotNeuralTangentKernel2020, leeWideNeuralNetworks2019}. For the contraction target we call this analog \emph{Kernel Contraction Matching} ($\kcm$): instead of learning a map, it predicts a kernel-weighted average of the training targets, in which each training point's influence decays smoothly with distance. $\kcm$ has no optimizer and no epoch axis, and runs on a CPU. Even so, on ADBench it matches established neural baselines~\cite{hanADBenchAnomalyDetection2022, ruffDeepOneClassClassification2018}, in line with the finding that non-deep methods are hard to beat on tabular data~\cite{grinsztajnWhyTreebasedModels2022}. And $\kcm$ cannot collapse: every training point's influence is local, so its prediction levels off as a test point leaves the data and the residual \emph{grows} on far anomalies instead of vanishing.

$\kcm$ shows a collapse-proof solution exists, but does not let a practitioner keep a neural backbone, so we turn it into a regularizer. The \emph{Kernel-Anchored Regularizer} ($\kar$) penalizes the deviation of the neural prediction from the $\kcm$ prediction across the input domain, confining the network to a narrow tube around the collapse-immune $\kcm$ solution. The tube width is the only knob, expressed directly in the units of the data, so there is no loss weight to tune. Within the tube the anomaly score stays close to the $\kcm$ score at every input, AUROC can move only slightly, and the collapsed constant solution falls outside the hypothesis class. Across collapse-prone ADBench datasets and three backbones, $\kar$ mitigates collapse and improves AUROC under prolonged training. The division of labour between the two is plain: \emph{$\kcm$ is the detector to deploy when scoring is the whole task; $\kar$ is for when a trainable backbone is a requirement rather than a choice}, because the pipeline fine-tunes the map downstream, needs a differentiable score, or must expose a learned representation. $\kar$'s benchmark-wide contribution is the removal of the collapse at no accuracy cost against its anchor, not an accuracy gain over $\kcm$. We make three contributions:
\begin{enumerate}
  \item \textbf{Convergence collapse:} we show that fixed-target detectors
  degrade as they converge, because the network matches its target off the
  normal distribution and the residual vanishes on anomalies; the guarantee
  is formal for backbones whose training covers the off-data region and
  empirical for the clean-input ones (\S\ref{sec:formal}). We give a
  taxonomy of anomaly detectors by their off-data score behaviour.
  \item \textbf{Kernel Contraction Matching:} we identify the kernel
  analog of any fixed-target detector as a closed-form, training-free,
  collapse-immune smoother that still matches established neural baselines on
  ADBench.
  \item \textbf{Kernel-Anchored Regularizer:} we anchor a neural backbone to
  $\kcm$ and show this removes the collapse across all $47$ ADBench datasets
  and three backbones at no accuracy cost against the anchor; on the
  collapse-prone subset it also improves on it.
\end{enumerate}

\section{Related Work}
\label{sec:related}

Anomaly detection spans a long line of method families~\cite{chandolaAnomalyDetectionSurvey2009}, and its deep-learning branch is by now broad~\cite{pangDeepLearningAnomaly2021}. Most deep detectors for tabular data learn a model of normal data and score a test point by how poorly that model accounts for it: reconstruction autoencoders by the reconstruction error~\cite{sattarovDiffusionScheduledDenoisingAutoencoders2025, ruffUnifyingReviewDeep2021}, one-class networks such as Deep SVDD by the distance to a compact centre~\cite{ruffDeepOneClassClassification2018}, self-supervised detectors by the consistency of a pretext prediction~\cite{shenkarAnomalyDetectionTabular2021, qiuNeuralTransformationLearning2022, yinMCMMaskedCell2023, thimonierIndividualInputDeep2024, wangUnsupervisedRepresentationLearning2020}. A more recent line scores the residual of a generative model: diffusion-based detectors read off a denoising signal~\cite{livernocheDIFFUSIONMODELINGANOMALY2024, daiUnsupervisedAnomalyDetection2025}, and flow-matching detectors regress a velocity field toward a fixed target, first for images~\cite{liHowWhyTaming2025, beizaeeREFLECTRectifiedFlows2025a} and then, through contraction matching, for tabular data~\cite{liScalableExplainableProvably2025}. All fit the fixed-target template, and we study a failure shared by the whole template rather than adding another detector to it.

Classical detectors predate these deep models and remain strong baselines on tabular benchmarks~\cite{hanADBenchAnomalyDetection2022, goldsteinComparativeEvaluationUnsupervised2016, boumanUnsupervisedAnomalyDetection2024}. Neighbourhood methods score a point from its proximity to nearby training samples: the $k$-nearest-neighbour distance~\cite{ramaswamyEfficientAlgorithmsMining2000}, the local outlier factor~\cite{breunigLOFIdentifyingDensitybased2000}, and kernel density estimation~\cite{langOneClassProcess2022}. One-class SVMs separate the data from the origin with a kernel of bounded reach~\cite{scholkopfSupportVectorMethod1999}; isolation forests~\cite{liuIsolationForest2008} and the copula-based COPOD~\cite{liCOPODCopulaBasedOutlier2020} rank a point by how extreme it is relative to the training sample as a whole. In each design the score is read from the stored training data rather than from a parametric map free to extrapolate; LUNAR~\cite{goodgeLUNARUnifyingLocal2022} folds the same neighbourhood structure back into a graph neural network.

That a deep detector can be harmed by additional training is already documented, along two strands. The first concerns contamination: when the unlabelled training data contains outliers, a discriminative network fits the inliers before the outliers~\cite{wangEffectiveEndtoendUnsupervised2019}, so detection peaks early and falls as the outliers are absorbed~\cite{nkashamaDeepLearningNetwork2024}, prompting label-free stopping criteria such as EntropyStop~\cite{huangEntropyStopUnsupervisedDeep2024} and GradStop~\cite{zhangGradStopExploringTraining} alongside evaluation protocols for unstable runs~\cite{periniUncertaintyawareEvaluationAuxiliary2024}. The second concerns reconstruction autoencoders trained on clean normal data: a sufficiently expressive autoencoder reconstructs some anomalies as accurately as normal points~\cite{gongMemorizingNormalityDetect2019, ghafourianTargetedCollapseRegularized2025}, met by autoencoder-specific responses such as a memory module~\cite{gongMemorizingNormalityDetect2019} or a penalty on the latent norm~\cite{ghafourianTargetedCollapseRegularized2025}. A third failure is mode collapse in Deep SVDD, averted by architectural restrictions on bias and activations~\cite{ruffDeepOneClassClassification2018} or, when those restrictions prove limiting, by explicit regularizers~\cite{chongSimpleEffectivePrevention2020}. We show the effect is a property of the entire fixed-target template, following from a single missing ingredient, locality, which also yields the cure.

Our treatment of that failure draws on the link between wide networks and kernel regression: in the lazy training regime the network behaves like a fixed kernel predictor governed by the neural tangent kernel~\cite{jacotNeuralTangentKernel2020, leeWideNeuralNetworks2019}, a picture normally used to study generalization, and recently to build kernel predictors for tabular models directly~\cite{miftachovInterpretableTabularFoundation2026}. We use it in the opposite direction, asking which kernel smoother a fixed-target detector corresponds to, and find that this smoother retains the locality the trained network loses. Kernel density estimation already turns a kernel smoother into an anomaly score~\cite{langOneClassProcess2022}; our construction applies the same principle to the contraction target rather than to the data density.

\section{Convergence Collapse}
\label{sec:collapse}

Fixed-target detectors degrade as they are optimized. This section makes the assertion precise.

\subsection{The fixed-target template}
\label{sec:template}

\begin{definition}[Fixed-target detector]\label{def:fixedtarget}
A \emph{fixed-target detector} is determined by a \emph{target}
$T:\R^D\to\R^D$, a fixed function of the input, specified in advance and
independently of the trained map, a \emph{training-input distribution} $\mu$
derived from the normal data, and a neural map $m_\theta:\R^D\to\R^D$. The map
is trained by the squared-error objective
\begin{equation}
  \mathcal{L}(\theta)=\E_{x\sim\mu}\,\|m_\theta(x)-T(x)\|^2,
  \label{eq:fixedloss}
\end{equation}
and a test point is scored by the \emph{residual}
\begin{equation}
  s(x)=\|m_\theta(x)-T(x)\|.
  \label{eq:residual}
\end{equation}
\end{definition}

Three detector classes instantiate Definition~\ref{def:fixedtarget}. The reconstruction autoencoder takes the identity target $T(x)=x$ and $\mu=p_{\mathrm{train}}$, with $m_\theta$ the encoder--decoder composition~\cite{ruffUnifyingReviewDeep2021, sattarovDiffusionScheduledDenoisingAutoencoders2025}. Contraction matching~\cite{liScalableExplainableProvably2025} trains a velocity field toward the negated input, $T(z)=-z$; we train it on the noised marginal $\mu$ of $z_t=(1-t)\epsilon+tx$ and read the residual at $t=1$. The one-step rectified-flow detector regresses a velocity along a straight path to a fixed endpoint~\cite{beizaeeREFLECTRectifiedFlows2025a}; we re-implement its rectified-flow objective on tabular ADBench data, as the original work targets brain-MRI anomalies.

What separates a detector that can be trained to convergence from one that cannot is a single property: what bounds the anomaly score \emph{away from the training data}, where anomalies lie. Table~\ref{tab:taxonomy} organizes detectors along this axis and is exhaustive over the families that populate the ADBench~\cite{hanADBenchAnomalyDetection2022} top tier.

\begin{table}[t]
  \centering
  \caption{A taxonomy of anomaly detectors by what constrains their off-data
  scores. Fixed-target methods (T3) can collapse.}
  \label{tab:taxonomy}
  \footnotesize
  \begin{tabular}{@{}>{\raggedright\arraybackslash}p{0.24\columnwidth}>{\raggedright\arraybackslash}p{0.5\columnwidth}c@{}}
  \toprule
  \textbf{Family} & \textbf{Off-data score is bounded by} & \textbf{Collapse?} \\
  \midrule
  \textbf{T1.}~Explicit locality
  & Nearby training points; far queries score high.
  & No \\
  \textbf{T2.}~Density-coupled target
  & A $p_{\mathrm{train}}$-dependent target; off-manifold scores stay
    non-trivial.
  & No \\
  \textbf{T3.}~Fixed target
  & An unconstrained residual to a fixed target; convergence can erase it.
  & Yes \\
  \bottomrule
  \end{tabular}
\end{table}

\textbf{T1, explicit locality.} The score at a query is a function of the training points in a bounded neighbourhood of it: $k$-nearest-neighbour distance, the local outlier factor~\cite{breunigLOFIdentifyingDensitybased2000}, kernel density estimation, one-class SVM~\cite{scholkopfSupportVectorMethod1999}, isolation forest~\cite{liuIsolationForest2008}, and LUNAR~\cite{goodgeLUNARUnifyingLocal2022}. A query with no training point nearby is anomalous by construction.

\textbf{T2, density-coupled target.} The method regresses a target that is itself a functional of $p_{\mathrm{train}}$: denoising score matching fits the score $\nabla\log p_\sigma$ of the noised density~\cite{songGenerativeModelingEstimating2020}, and energy-based detectors fit an energy whose negatives depend on the model itself~\cite{yoonEnergyBasedModelsAnomaly2023}. Because the target encodes where the density's mass lies, its optimum stays non-trivial off-manifold.

\textbf{T3, fixed target.} The detectors of Definition~\ref{def:fixedtarget}. The target is fixed in advance and carries no information about where the normal data lies, so the residual off the data is whatever the network chooses to extrapolate. T3 is the only family in Table~\ref{tab:taxonomy} without an off-data floor on its score, and, we show next, the only one that collapses.

\subsection{The collapse, formally}
\label{sec:formal}

\begin{proposition}[Convergence collapse]\label{prop:collapse}
Let a fixed-target detector have target $T$, training-input distribution $\mu$,
and a hypothesis class $\mathcal{F}\subseteq C(\R^D,\R^D)$ of continuous maps
with $T\in\mathcal{F}$. Let $m^\star\in\arg\min_{m\in\mathcal{F}}\mathcal{L}(m)$
attain $\mathcal{L}(m^\star)=0$. Then $m^\star\equiv T$ on
$\operatorname{supp}(\mu)$, and the residual score \eqref{eq:residual} satisfies
$s\equiv 0$ on $\operatorname{supp}(\mu)$. Every test point in
$\operatorname{supp}(\mu)$, normal or anomalous, then receives the identical
score $0$, so the detector cannot rank anomalies above normal points: its AUROC
on $\operatorname{supp}(\mu)$ is $1/2$.
\end{proposition}

\begin{proof}
$\mathcal{L}(m^\star)=\int\|m^\star-T\|^2\,\mathrm{d}\mu=0$ forces
$m^\star=T$ on a set $A$ with $\mu(A^c)=0$. The set $A$ is dense in
$\operatorname{supp}(\mu)$: every open ball around a support point has positive
$\mu$-measure and so cannot lie inside the null set $A^c$, hence meets $A$.
Both $m^\star$ and $T$ are continuous, and two continuous maps that agree on a
dense subset agree on its closure; therefore $m^\star\equiv T$ on
$\operatorname{supp}(\mu)$, and $s=\|m^\star-T\|\equiv 0$ there. A score
constant on a set ranks the points of that set uniformly at random, which gives
AUROC $1/2$.
\end{proof}

\paragraph{Scope of Proposition~\ref{prop:collapse}.}
The collapse rests on two premises: the target $T$ is fixed in advance, and $\mathcal{F}$ can realize $T$ on the region that matters. The support of $\mu$ decides which backbones the guarantee covers. \emph{(i) Formal.} Contraction matching and the rectified-flow detectors train on a \emph{noised} marginal, so $\operatorname{supp}(\mu)$ covers the off-manifold region where anomalies lie and Proposition~\ref{prop:collapse} applies with no further assumption: the optimum drives the residual to zero on anomalies as exactly as on normal points. \emph{(ii) Empirical.} The reconstruction autoencoder and Deep SVDD~\cite{ruffDeepOneClassClassification2018} (target a constant $c$) train on the clean normal cloud, so Proposition~\ref{prop:collapse} pins $s\equiv 0$ only there; off it the residual depends on how $\mathcal{F}$ extrapolates the near-identity (resp.\ near-constant) map. For these backbones the off-manifold part of the claim is empirical, not proved: the smooth networks used in practice extrapolate their on-data fit, and Table~\ref{tab:breadth} shows the same decay on both. A regularity assumption on the realized map off $\operatorname{supp}(\mu)$, a Lipschitz bound on $m^\star-T$, is the hypothesis that would make case (ii) formal; we do not assume it.

Figure~\ref{fig:collapse} shows the collapse along training: in the lower panels, nine detectors trained on the ADBench \texttt{7\_Cardiotocography} dataset for $3000$ epochs and $3$ seeds all reach an early AUROC peak and then degrade, although the training loss keeps decreasing. The peak-to-final AUROC drop is $0.128$ for the autoencoder, $0.051$ for contraction matching, and $0.017$ for the rectified-flow detector. Table~\ref{tab:breadth} gives the dataset-level view across six residual-score detectors and six ADBench datasets; the gap is largest on mid-difficulty datasets and is not tied to a particular backbone or dataset.

\begin{table}[t]
  \centering
  \caption{Peak-to-last-epoch AUROC gap, averaged over $3$ seeds after $100$
  epochs of vanilla training. Positive values mean that training past the early
  peak hurts detection.}
  \label{tab:breadth}
  \footnotesize
  \setlength{\tabcolsep}{2pt}
  \begin{tabular}{@{}lcccccc@{}}
  \toprule
  \textbf{Detector} & \texttt{vert.} & \texttt{annth.} & \texttt{Pima}
  & \texttt{Ionos.} & \texttt{WDBC} & \texttt{cardio} \\
  \midrule
  Denoising autoencoder~\cite{ruffUnifyingReviewDeep2021}
    & $.022$ & $.050$ & $.039$ & $.004$ & $.007$ & $.070$ \\
  Targeted-collapse autoencoder \cite{ghafourianTargetedCollapseRegularized2025}
    & $.000$ & $.008$ & $.018$ & $.005$ & $.007$ & $.065$ \\
  Deep SVDD~\cite{ruffDeepOneClassClassification2018}
    & $.019$ & $.119$ & $.019$ & $.003$ & $.005$ & $.034$ \\
  Random-distance prediction \cite{wangUnsupervisedRepresentationLearning2020}
    & $.106$ & $.030$ & $.048$ & $.011$ & $.013$ & $.011$ \\
  Rectified-flow correction~\cite{beizaeeREFLECTRectifiedFlows2025a}
    & $.018$ & $.110$ & $.091$ & $.004$ & $.005$ & $.018$ \\
  Flow-matching detector~\cite{liScalableExplainableProvably2025}
    & $.218$ & $.047$ & $.042$ & $.072$ & $.002$ & $.015$ \\
  \bottomrule
  \end{tabular}
\end{table}

Two observations rule out a simple early-stopping interpretation. The best epoch is unstable: in Figure~\ref{fig:collapse} it varies by more than two orders of magnitude and depends on the seed. The optimized loss is also misleading, since it decreases monotonically while AUROC follows an inverted-U curve, so monitoring the training loss pushes the detector toward collapse.

This is not ordinary overfitting. Overfitting is a finite-sample failure, the empirical objective improving while population performance worsens; convergence collapse is a population-objective failure, the optimum itself having no anomaly signal, which Proposition~\ref{prop:collapse} proves independently of train--test generalization. The experiments agree: the held-out-normal residual tracks the training residual with a median difference of $0.01$ along the contraction backbone's trajectory, far below the normal--anomaly gap the detector must preserve.

Proposition~\ref{prop:collapse} also names the only two ways out, one per premise. Negating the first gives a density-coupled target, the route of family T2. Negating the second restricts $\mathcal{F}$ so that the constant target is unreachable on the region that matters; a finite-bandwidth local smoother does exactly this, the route this paper takes next.

\section{Kernel Contraction Matching}
\label{sec:kcm}

Section~\ref{sec:formal} closed by naming two ways out of the collapse, one for each premise of Proposition~\ref{prop:collapse}. This section follows the second: restrict the map so that the fixed target cannot be reproduced away from the data. The restriction needs no new method. We take a fixed-target detector and replace its trained map by the kernel-regression estimator of the same target. The result is a closed-form detector we call Kernel Contraction Matching (\kcm).

\begin{figure}[t]
  \centering
  \includegraphics[width=0.75\columnwidth]{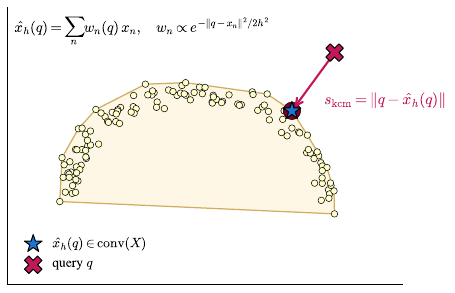}
  \caption{\kcm{} at a query $q$: the prediction $\hat{x}_h(q)=\sum_n w_n(q)\,x_n$ with $w_n\!\propto\!e^{-\|q-x_n\|^2/2h^2}$ is a kernel-weighted average of the training normals (point size and colour scale with the weight). The reconstruction lies in $\operatorname{conv}(X)$ (Proposition~\ref{prop:immune}), and the anomaly score $s_{\kcm}(q)=\|q-\hat{x}_h(q)\|$ is the residual arrow.}
  \Description{Schematic of KCM at a single query: training points sit inside their convex hull (shaded polygon); point sizes and colours encode their kernel weight to a query (red X) outside the hull; a blue star inside the hull marks the kernel-weighted-average reconstruction; an arrow from the query to the reconstruction is labelled with the KCM score, the residual norm.}
  \label{fig:kcm-schematic}
\end{figure}

\paragraph{The kernel-regression analog.}

A fixed-target detector fits a neural map $m_\theta$ to a target $T$ on inputs drawn from $\mu$ (Definition~\ref{def:fixedtarget}). Its \emph{kernel-regression analog} keeps the target and the residual score of Eq.~\eqref{eq:residual} and changes only the map. In place of a trained network it predicts a kernel-weighted average of the training targets. Let $X=\{x_n\}_{n=1}^{N}\subset\R^D$ be the training normals and let $q\in\R^D$ be a query point. With an RBF kernel $K_h(q,x)=\exp(-\|q-x\|^2/2h^2)$ of bandwidth $h$, define the weights and the estimator
\begin{equation}
  w_n(q)=\frac{K_h(q,x_n)}{\sum_{m=1}^{N}K_h(q,x_m)},
  \qquad
  f_{\kcm}(q)=\sum_{n=1}^{N}w_n(q)\,T(x_n).
  \label{eq:kcm}
\end{equation}
This is the Nadaraya--Watson smoother of the target values, a classical estimator. For the contraction target $T(x)=-x$ it reads $f_{\kcm}(q)=-\hat{x}_h(q)$, where
\begin{equation}
  \hat{x}_h(q)=\sum_{n=1}^{N}w_n(q)\,x_n
  \label{eq:recon}
\end{equation}
is the kernel reconstruction of $q$ from the training normals. The anomaly score is the residual of Definition~\ref{def:fixedtarget},
\begin{equation}
  s_{\kcm}(q)=\bigl\|f_{\kcm}(q)-T(q)\bigr\|=\bigl\|q-\hat{x}_h(q)\bigr\|.
  \label{eq:kcmscore}
\end{equation}

The same substitution sends the other members of the family to the same place: for the identity target of the autoencoder, $T(x)=x$, the analog is $f_{\kcm}(q)=\hat{x}_h(q)$ and the score is again Eq.~\eqref{eq:kcmscore}, the fixed target cancelling in the residual. \kcm{} is the analog of the fixed-target family, not of one detector in it.

Equation~\eqref{eq:kcm} is a \emph{linear smoother} of the fixed target: a prediction linear in the training targets $T(x_n)$, with input-dependent weights. The estimator is textbook, and we do not claim it as new. What we claim is that it belongs to the same estimator class a trained fixed-target detector occupies in the wide-network limit, that unlike the trained detector it cannot collapse, and that it anchors the regularizer of Section~\ref{sec:kar}; the rest of this section establishes the first two.

\paragraph{The lazy limit of a fixed-target detector.}

The lazy training regime turns the word ``analog'' into a theorem: a wide fixed-target detector trained by gradient flow is itself a linear smoother of its targets. Let $m_\theta$ be a wide network at the standard neural-tangent parameterization, trained on the squared-error objective of Eq.~\eqref{eq:fixedloss} over inputs $x_1,\dots,x_N$ sampled from $\mu$, with targets stacked as $\mathbf{T}=(T(x_1),\dots,T(x_N))$. Write $\Theta$ for the network's neural tangent kernel, $\Theta(X,X)$ for its Gram matrix on the training inputs, positive definite for distinct inputs and acting identically on each output coordinate, and $\Theta(q,X)$ for the row of kernel values between a query $q$ and those inputs.

\begin{theorem}[Lazy limit of a fixed-target detector]\label{thm:lazy}
In the infinite-width limit, and from an initialization with zero output, the
detector trained by gradient flow, the continuous-time limit of gradient
descent at learning rate $\eta$, realizes after training time $\tau$ the
predictor
\begin{equation}
  m_\tau(q)=\Theta(q,X)\,\Theta(X,X)^{-1}
          \bigl(I-e^{-\eta\tau\,\Theta(X,X)}\bigr)\,\mathbf{T}.
  \label{eq:lazy}
\end{equation}
As $\tau\to\infty$ the matrix exponential vanishes and $m_\tau$ converges to the
tangent-kernel interpolant
$m_\infty(q)=\Theta(q,X)\,\Theta(X,X)^{-1}\mathbf{T}$, which fits the targets
exactly on the training inputs.
\end{theorem}

\begin{proof}
See Appendix~\ref{app:proofs}.
\end{proof}

\begin{remark}[Early stopping is regularization]\label{rem:ridge}
The operator $I-e^{-\eta\tau\,\Theta(X,X)}$ in Eq.~\eqref{eq:lazy} is a spectral
filter. On an eigencomponent of $\Theta(X,X)$ with eigenvalue $\sigma$ it
applies the factor $1-e^{-\eta\tau\sigma}$, near one for large $\sigma$ and near
zero for small $\sigma$. It belongs to the same family of monotone shrinkage
filters as the ridge filter $\sigma/(\sigma+\lambda)$, with
$\lambda\approx 1/(\eta\tau)$: more training is weaker regularization. The per-dataset stopping epoch that fixed-target detectors depend on
(\S\ref{sec:collapse}) is, read through the lazy limit, an implicit and
unreported choice of this regularization.
\end{remark}

Theorem~\ref{thm:lazy} places the trained detector and \kcm{} in one class: both are linear smoothers of the fixed target, $\sum_n\alpha_n(q)\,T(x_n)$ with input-dependent weights, the former with the tangent-kernel weights of Eq.~\eqref{eq:lazy}, the latter with the Nadaraya--Watson weights of Eq.~\eqref{eq:kcm}. \kcm{} is not the lazy limit: one is regularized by the training time on the neural tangent kernel, the other by the bandwidth of a finite-bandwidth RBF. The theorem establishes something weaker and more useful than an identity: the two are the same kind of estimator, and which smoother in that class one uses is what the rest of this section turns on.

\paragraph{Why build on a limit that excludes feature learning?}
The infinite-width theory behind Theorem~\ref{thm:lazy} assumes no features are learned, while large finite networks are valued exactly because they warp their empirical tangent kernel toward the task, so a method drawn from the lazy limit can look self-contradictory. Feature learning, however, is the mechanism that lets a finite network match the fixed target \emph{off} the data, and Proposition~\ref{prop:collapse} says that this is what collapse is. A non-feature-learning prior is not an oversight here; it supplies the one thing the feature-learning regime cannot, an off-hull floor on the score. \kar{} (Section~\ref{sec:kar}) does not claim the lazy limit describes the trained network; it uses the kernel smoother as a constraint set, and its radius prices the tension: $\rho=0$ is the anchor with no feature learning, $\rho\to\infty$ is full feature learning and full collapse. How much feature learning survives inside the tube is not measured here; an empirical-tangent-kernel drift measurement along training would quantify it.

\paragraph{Bandwidth by closed-form leave-one-out.}

\kcm{} has one parameter, the bandwidth $h$. We set it by leave-one-out cross-validation on the training normals. For a candidate $h$, write $e_n(h)=\bigl\|x_n-\hat{x}_h^{(-n)}(x_n)\bigr\|^2$ for the squared leave-one-out reconstruction error of $x_n$, where $\hat{x}_h^{(-n)}(x_n)$ is the estimator of Eq.~\eqref{eq:recon} evaluated at $x_n$ with $x_n$ removed from the kernel sum and the weights renormalized over the remaining points. A semi-supervised training set is rarely perfectly clean: a few unlabeled anomalies sit among the normals, and they carry the largest leave-one-out errors. We therefore select $h$ by the \emph{trimmed} mean of these errors, discarding the largest fraction $\beta$ before averaging. With $K=\lceil(1-\beta)N\rceil$ and $e_{(1)}(h)\le\dots\le e_{(N)}(h)$ the sorted errors,
\begin{equation}
  h^\star=\arg\min_{h}\;\frac{1}{K}\sum_{n=1}^{K}e_{(n)}(h).
  \label{eq:loo}
\end{equation}
We fix $\beta=0.1$. Trimming keeps the contaminating anomalies from pulling the bandwidth selector; $\beta=0$ recovers the plain leave-one-out mean. Computing $e_n(h)$ for all $n$ is a single operation: set the $n$-th entry of each kernel row to zero before renormalizing. The pairwise distances among the training normals are computed once, and the objective is then one $O(N^2)$ pass for each bandwidth on a fixed logarithmic grid. Equation~\eqref{eq:loo} uses the training normals alone, with no anomaly labels, no validation split, and no gradient steps: the regularization Remark~\ref{rem:ridge} leaves implicit in a stopping epoch is fixed here at one explicit value. Algorithm~\ref{alg:kcm} states the procedure in full.

\begin{algorithm}[t]
\caption{Kernel Contraction Matching}
\label{alg:kcm}
\begin{algorithmic}[1]
\REQUIRE training normals $X=\{x_n\}_{n=1}^N$, bandwidth grid $\mathcal{H}$,
         trim fraction $\beta$
\STATE \textbf{Fit (closed form, no labels):}
\STATE compute the kernel matrix $K_h$ on $X$ for each $h\in\mathcal{H}$
\STATE $h^\star\gets\arg\min_{h\in\mathcal{H}}$ trimmed-mean$_\beta\bigl\{\,e_n(h)\,\bigr\}_{n=1}^{N}$
       \hfill $\triangleright$ Eq.~\eqref{eq:loo}
\STATE \textbf{Score a query $q$:}
\RETURN $s_{\kcm}(q)=\|q-\hat{x}_{h^\star}(q)\|$
       \hfill $\triangleright$ Eq.~\eqref{eq:kcmscore}
\end{algorithmic}
\end{algorithm}

\paragraph{KCM cannot collapse.}

\kcm{} cannot undergo a convergence collapse: it is a closed form with no optimization trajectory to degrade along, and its score has a structural floor outside the training hull.

\begin{proposition}[Off-hull score floor]\label{prop:immune}
For every bandwidth $h>0$ and every query $q\in\R^D$, the kernel reconstruction
$\hat{x}_h(q)$ lies in the convex hull $\operatorname{conv}(X)$ of the training
normals, so the \kcm{} score is bounded below by the distance from the query to
that hull,
\begin{equation}
  s_{\kcm}(q)=\|q-\hat{x}_h(q)\|\;\ge\;\operatorname{dist}\!\bigl(q,\operatorname{conv}(X)\bigr).
  \label{eq:floor}
\end{equation}
The bound is zero only for $q$ inside the hull and grows without bound as $q$
leaves it: the \kcm{} score is strictly positive outside the convex hull of the
training normals.
\end{proposition}

\begin{proof}
The RBF kernel is positive, so the weights $w_n(q)$ of Eq.~\eqref{eq:kcm} are
non-negative and sum to one; $\hat{x}_h(q)=\sum_n w_n(q)\,x_n$ is then a convex
combination of $x_1,\dots,x_N$ and lies in $\operatorname{conv}(X)$. Every
$c\in\operatorname{conv}(X)$ obeys
$\|q-c\|\ge\operatorname{dist}(q,\operatorname{conv}(X))$ by definition of the
distance to a set, and $c=\hat{x}_h(q)$ gives Eq.~\eqref{eq:floor}; the hull is
bounded, so its distance function is unbounded.
\end{proof}

Theorem~\ref{thm:lazy} explains the contrast. The trained detector and \kcm{} are both linear smoothers of the target; the lazy-limit weights are signed and need not sum to one, so training off the data collapses the score (Proposition~\ref{prop:collapse}), while \kcm{}'s non-negative weights confine $\hat{x}_h(q)$ to $\operatorname{conv}(X)$ and floor it. This places \kcm{} in family T1 of Table~\ref{tab:taxonomy}: it is the collapse-proof detector Section~\ref{sec:collapse} called for, and the anchor the next section regularizes a neural backbone toward.

\section{KCM is Competitive on ADBench}
\label{sec:kcm-results}

Section~\ref{sec:kcm} introduced \kcm{} as a closed-form detector and proved it cannot collapse. A separate question is whether a detector built to avoid a failure mode also detects anomalies well. We answer it on ADBench: \kcm{} is level with the strongest label-free detectors on the benchmark, and it gets there with no training, no GPU, and no labels.

\paragraph{Evaluation protocol.}
We follow the semi-supervised protocol. For each dataset \kcm{} is fit on a sample of normal points and scores a held-out mixture of normal and anomalous points; the fit is the procedure of Section~\ref{sec:kcm} unchanged ($\beta=0.1$), no anomaly labels enter it, and features are standardized per coordinate. We report mean AUROC over the $47$ datasets, averaged over three seeds. \kcm{}'s numbers are our runs; the baseline numbers come from the TCCM benchmark~\cite{liScalableExplainableProvably2025}, which reports per-dataset AUROC for its $44$ baselines plus TCCM on the same $47$ ADBench datasets.

\paragraph{Datasets.}
ADBench~\cite{hanADBenchAnomalyDetection2022} collects $47$ tabular datasets from healthcare, finance, image-derived, and network domains, spanning nearly four orders of magnitude in sample size and up to more than a thousand features.

\paragraph{Baseline methods.}
We compare against the ADBench suite, $45$ detectors that together with \kcm{} give $46$ label-free methods and cover the families of the taxonomy in Table~\ref{tab:taxonomy}. They are classical density and distance detectors (kernel density estimation, $k$-nearest neighbours, the local outlier factor~\cite{breunigLOFIdentifyingDensitybased2000}, isolation forest~\cite{liuIsolationForest2008}, the one-class SVM~\cite{scholkopfSupportVectorMethod1999}, COPOD~\cite{liCOPODCopulaBasedOutlier2020}, Gaussian mixtures), deep detectors (the autoencoder~\cite{hanADBenchAnomalyDetection2022}, Deep SVDD~\cite{ruffDeepOneClassClassification2018}, LUNAR~\cite{goodgeLUNARUnifyingLocal2022}, the diffusion-time estimators DTE-NP (non-parametric) and DTE-Cat (categorical)~\cite{livernocheDIFFUSIONMODELINGANOMALY2024}), and the fixed-target detectors of \S\ref{sec:collapse}, among them TCCM~\cite{liScalableExplainableProvably2025}. We state the result as three claims.

\begin{figure*}[t]
  \centering
  \includegraphics[width=0.86\textwidth]{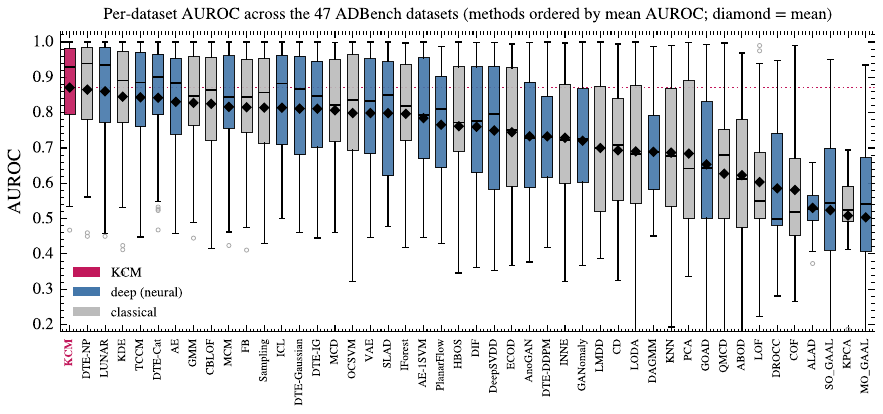}
  \caption{Per-dataset AUROC of the $46$ label-free detectors over the $47$ ADBench datasets, ordered by mean (diamond). \kcm{} (red) has the highest mean AUROC, level with DTE-NP and LUNAR (the only two methods within $0.012$).}

  \Description{Per-method AUROC distribution over the 47 ADBench datasets, methods ordered left-to-right by mean AUROC. Each method is a box-and-whisker with the mean as a diamond. KCM, drawn red on the far left, has the highest mean AUROC; DTE-NP and LUNAR are next; the bulk of classical detectors (grey) and deep detectors (blue) follow.}  \label{fig:kcm-boxplot}
\end{figure*}

\paragraph{(R1) KCM ranks first by mean AUROC.}
Across the $47$ datasets \kcm{} reaches a mean AUROC of $0.871$, the highest of the $46$ label-free detectors. The runners-up are the non-parametric DTE-NP at $0.865$ and the graph detector LUNAR at $0.860$; the strongest classical baseline, kernel density estimation, follows at $0.845$, and TCCM at $0.843$.\footnote{TCCM is re-run under our harness and without epoch fine-tuning; the TCCM paper reports $0.867$ on the same $47$ datasets.} Figure~\ref{fig:kcm-boxplot} plots the per-dataset AUROC of every method ordered by mean. \kcm{} sits at the left with the highest mean and no low tail: rarely the single best detector on a dataset ($5$ of $47$) but almost never far down the list (top five on $21$ of $47$), it leads the aggregate by consistency rather than by spiking on a few datasets. The trimmed selector contributes a small part of the lead: the untrimmed mean gives $0.866$, also first, and the trim at $\beta=0.1$ adds $0.005$. On mean AUPR \kcm{} ranks third at $0.644$, behind LUNAR and DTE-NP; it leads the AUROC leaderboard and stays in the top tier on AUPR.

\paragraph{(R2) KCM is level with the two strongest deep detectors and clear of
the rest.} A mean of $0.871$ against $0.865$ is too small to carry an ordering on its own: the seed-to-seed standard deviation is about $10^{-3}$, and the leads over DTE-NP ($+0.006$) and LUNAR ($+0.011$) sit within the run-to-run spread, so \kcm{} is in practice indistinguishable from the two strongest deep detectors. The margins over the methods below them are larger and consistent: $+0.026$ over kernel density estimation, $+0.028$ over TCCM, $+0.040$ over the autoencoder, and $+0.074$ over isolation forest. The Nemenyi critical-difference diagram in Figure~\ref{fig:kcm-cd} ($\mathrm{CD}=2.43$) gives the same picture as a ranking: it places \kcm{}, DTE-NP, and LUNAR in one group, clear of the rest.

\begin{figure}[t]
  \centering
  \includegraphics[width=0.84\columnwidth]{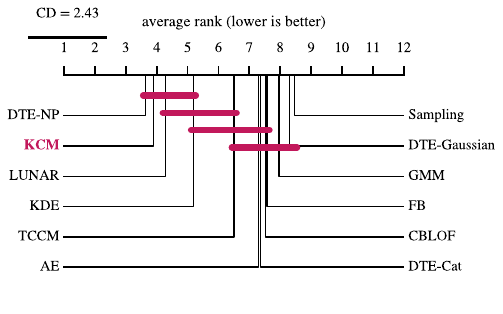}
  \caption{Nemenyi critical-difference diagram, top twelve methods ($\mathrm{CD}=2.43$); methods joined by a bar are not significantly different. \kcm{} shares the leading group with DTE-NP and LUNAR.}
  \Description{Nemenyi critical-difference diagram over the top-twelve detectors on the 47 ADBench datasets, with critical difference 2.43. Methods are arranged along a horizontal mean-rank axis; methods whose ranks differ by less than the critical difference are joined by a thick horizontal bar. KCM, DTE-NP and LUNAR are in the leading group.}
  \label{fig:kcm-cd}
\end{figure}

\paragraph{(R3) KCM reaches this position with no training.}
\kcm{} trains nothing: no gradient steps, no epochs, no GPU, no early-stopping search, and no labels, only the closed-form fit of Section~\ref{sec:kcm}. For $N$ training normals and $M$ test points it costs $O(N^2D)$ to select the bandwidth and $O(NMD)$ to score; for large $N$ a uniform reference subsample of $N_0$ points (we use $N_0=2000$) bounds the fit at $O(N_0^2D)$. Table~\ref{tab:kcm-speed} reports single-CPU-core wall-clock over the $47$ datasets. \kcm{}'s fit has a median of $0.33$ seconds and reaches at most $1.2$, below LUNAR ($3.5$ s median, worst $1018$), the autoencoder ($1.6$ s) and TCCM ($1.1$ s, worst $94$); its median inference is $0.06$ seconds, where the non-parametric DTE-NP pays a median of $0.30$ rising past $700$ on the largest datasets. The two methods tied with \kcm{} each carry a cost it does not: LUNAR in training, DTE-NP at inference. The reference cap is priced, not assumed: on the three largest ADBench datasets (\texttt{cover}, \texttt{fraud}, \texttt{census}), refitting at $N\in\{500,\dots,20{,}000\}$ training normals puts the uncapped fit on the quadratic branch ($265$ to $1030$ s at $N{=}20{,}000$; single core, median of three seeds) while the capped fit stays below $9$ s at every $N$. And the subsample does not buy the accuracy: at $N{=}20{,}000$ the cap moves mean AUROC by $+0.005$ (\texttt{cover}), $-0.001$ (\texttt{fraud}) and $-0.005$ (\texttt{census}), positive favouring the uncapped fit, and stays within $\pm0.01$ at every measured $N$.

\begin{table}[t]
  \centering
  \caption{Per-detector wall-clock over the $47$ ADBench datasets, one CPU core: median fit, median inference, worst-case total; cells shaded $\propto\log$ within each column. KDE and DTE-NP did not complete on $5$ and $1$ datasets (medians cover completed runs). All timings are author measurements.}
  \label{tab:kcm-speed}
  \footnotesize
  \setlength{\tabcolsep}{5pt}
  \begin{tabular}{@{}lrrr@{}}
  \toprule
  \textbf{Method} & \textbf{Fit (s)} & \textbf{Infer.\ (s)} & \textbf{Worst (s)} \\
  \midrule
  kNN & \cellcolor{orange!30}$0.021$ & \cellcolor{orange!28}$0.014$ & \cellcolor{orange!52}$280$ \\
  IForest & \cellcolor{orange!37}$0.062$ & \cellcolor{orange!23}$0.009$ & \cellcolor{orange!10}$2.60$ \\
  ECOD & \cellcolor{orange!43}$0.156$ & \cellcolor{orange!23}$0.009$ & \cellcolor{orange!34}$38.0$ \\
  DTE-NP & \cellcolor{orange!10}$0.001$ & \cellcolor{orange!64}$0.302$ & \cellcolor{orange!61}$771$ \\
  KDE & \cellcolor{orange!43}$0.152$ & \cellcolor{orange!60}$0.213$ & \cellcolor{orange!64}$1046$ \\
  \textbf{KCM} & \cellcolor{orange!48}$0.331$ & \cellcolor{orange!45}$0.060$ & \cellcolor{orange!34}$38.5$ \\
  TCCM & \cellcolor{orange!56}$1.09$ & \cellcolor{orange!10}$0.003$ & \cellcolor{orange!42}$94.4$ \\
  AE & \cellcolor{orange!59}$1.57$ & \cellcolor{orange!28}$0.014$ & \cellcolor{orange!36}$48.5$ \\
  LUNAR & \cellcolor{orange!64}$3.46$ & \cellcolor{orange!35}$0.025$ & \cellcolor{orange!64}$1018$ \\
  \bottomrule
  \end{tabular}
\end{table}

\section{Kernel-Anchored Regularization}
\label{sec:kar}

\kcm{} cannot collapse, at a price: it has no trainable parameters, so a dataset cannot improve it. The vanilla fixed-target detector is trainable, and Section~\ref{sec:collapse} showed that training is what destroys it. This section keeps the trainable neural map and removes the failure by constraining it to a fixed radius of the kernel-weighted average, a regularizer that is collapse-proof for the same structural reason \kcm{} is.

\subsection{Definition}
\label{sec:kar-def}

\kcm{} replaced the neural map $m_\theta$ of Definition~\ref{def:fixedtarget} by the kernel-weighted average $f_{\kcm}$ of Eq.~\eqref{eq:kcm}. Kernel-anchored regularization, \kar{}, keeps the neural map and adds a penalty that holds it close to $f_{\kcm}$.

Write the deviation of the prediction from the kernel-weighted average at a point $x$ as
\begin{equation}
  \delta_\theta(x)=\bigl\|m_\theta(x)-f_{\kcm}(x)\bigr\|.
  \label{eq:kar-dev}
\end{equation}
The \emph{kernel-anchored regularizer} $R_\rho$ penalizes this deviation:
\begin{equation}
  \min_\theta\ \E_x\bigl\|m_\theta(x)-T(x)\bigr\|^2+R_\rho(m_\theta),
  \qquad
  R_\rho(m_\theta)=\sum_x P_\rho\bigl(\delta_\theta(x)\bigr),
  \label{eq:kar-obj}
\end{equation}
with $P_\rho$ the deadzone penalty
\begin{equation}
  P_\rho(\delta)=0 \ \text{ for } \delta<\rho,
  \qquad
  P_\rho(\delta)=+\infty \ \text{ for } \delta\ge\rho.
  \label{eq:kar-deadzone}
\end{equation}
$R_\rho$ is a regularizer in the ordinary sense, a nondecreasing penalty on the size of the deviation field $\delta_\theta$; the deadzone $P_\rho$ leaves the prediction free within the anchor radius $\rho$ and forbids it past that radius. The one parameter is $\rho$, measured in the units of the data and directly comparable to the \kcm{} residual scale: at $\rho=0$ the prediction is pinned to $f_{\kcm}$ and \kar{} is \kcm{}; as $\rho\to\infty$ the penalty vanishes and \kar{} is the vanilla detector.

\paragraph{Constrained form.}
Because $P_\rho$ is infinite outside the anchor radius, $R_\rho$ is the indicator function of the set
\begin{equation}
  \mathcal{C}=\bigl\{\,m:\ \|m(x)-f_{\kcm}(x)\|<\rho \ \text{ for all } x\,\bigr\},
  \label{eq:kar-tube}
\end{equation}
the tube of radius $\rho$ around $f_{\kcm}$. The penalized objective is then identically the constrained problem $\min_\theta\E_x\|m_\theta-T\|^2$ subject to $m_\theta\in\mathcal{C}$; Proposition~\ref{prop:kar-equiv} records the equivalence.

\paragraph{The constraint is over the whole input space.}
The membership $\|m(x)-f_{\kcm}(x)\|<\rho$ in Eq.~\eqref{eq:kar-tube} is asked at every $x\in\R^D$, not only at the training points, and this carries the method: the collapse is an off-distribution event (Proposition~\ref{prop:collapse}), and a penalty summed over training inputs is blind there, exactly as the fixed-target loss is. $R_\rho$ is therefore a constraint on the function $m_\theta$ across its domain, and it is enforced as one.

\paragraph{Realization.}
The membership $m_\theta\in\mathcal{C}$ is enforced exactly, at every $x$, by writing the prediction as the kernel-weighted average plus a bounded correction:
\begin{equation}
  m_\theta(x)=f_{\kcm}(x)+g_\theta(x),
  \qquad
  g_\theta(x)=\rho\,\frac{u_\theta(x)}{1+\|u_\theta(x)\|},
  \label{eq:kar-arch}
\end{equation}
where $u_\theta$ is an unconstrained multilayer perceptron. The radial squash gives $\|g_\theta(x)\|<\rho$ for every $x$, so $m_\theta$ ranges over exactly the open tube $\mathcal{C}$, and the correction $g_\theta$ is trained under the fixed-target loss $\E_x\|m_\theta-T\|^2$. This is \kar{}: the regularizer is not a term added to the loss but the hypothesis class the loss is minimized over, which is what lets a constraint over all of $\R^D$ be imposed without an intractable optimization.

\subsection{Propositions}
\label{sec:kar-props}

\begin{proposition}[Penalty and constraint coincide]\label{prop:kar-equiv}
For any objective $L$,
$\inf_m\,[\,L(m)+R_\rho(m)\,]=\inf_{m\in\mathcal{C}}L(m)$, with the same
infimum.
\end{proposition}
\begin{proof}
See Appendix~\ref{app:proofs}.
\end{proof}

\noindent Write $r_{\kcm}(x)=x-\hat{x}_h(x)$ for the \kcm{} residual, with $\|r_{\kcm}(x)\|=s_{\kcm}(x)$ (Eq.~\eqref{eq:kcmscore}).

\begin{proposition}[Score floor]\label{prop:kar-floor}
For the contraction target $T(x)=-x$, the \kar{} score satisfies, at every $x$,
\begin{equation}
  \begin{aligned}
  s(x)&=\bigl\|m_\theta(x)-T(x)\bigr\|=\bigl\|r_{\kcm}(x)+g_\theta(x)\bigr\|,\\
  &\quad\ \bigl|s(x)-s_{\kcm}(x)\bigr|<\rho.
  \end{aligned}
  \label{eq:kar-floor}
\end{equation}
By Proposition~\ref{prop:immune},
$s(x)>\operatorname{dist}(x,\operatorname{conv}(X))-\rho$, so $s$ cannot vanish
sufficiently far outside the convex hull of the training normals: the collapsed
solution $s\equiv0$ is outside the hypothesis class for every finite $\rho$.
\end{proposition}
\begin{proof}
See Appendix~\ref{app:proofs}.
\end{proof}

\noindent The two rankings agree on every point pair whose \kcm{} scores differ by more than $2\rho$, so the peak-to-final AUROC gap is controlled by the pair mass below $2\rho$: it vanishes as $\rho\to0$ and grows with the radius, as the ablation below measures.

\begin{proposition}[One detector across backbones]\label{prop:kar-backbone}
Instantiate \kar{} on a fixed-target backbone with target $T_B$ by anchoring at
$A_B=T_B+r_{\kcm}$ and predicting $m_\theta=A_B+g_\theta$. Then the training loss
and the anomaly score,
\begin{equation}
  \begin{aligned}
  \E_x\bigl\|m_\theta(x)-T_B(x)\bigr\|^2
    &=\E_x\bigl\|r_{\kcm}(x)+g_\theta(x)\bigr\|^2,\\
  \bigl\|m_\theta(x)-T_B(x)\bigr\|
    &=\bigl\|r_{\kcm}(x)+g_\theta(x)\bigr\|,
  \end{aligned}
  \label{eq:kar-backbone}
\end{equation}
are both independent of $T_B$: \kar{} is the one detector of Proposition~\ref{prop:kar-floor} on any fixed-target backbone.
\end{proposition}
\begin{proof}
See Appendix~\ref{app:proofs}.
\end{proof}

\noindent For contraction matching ($T=-x$) the anchor $A_B=-\hat{x}_h$ is the kernel-weighted average $f_{\kcm}$ itself; for the autoencoder ($T=+x$) and Deep SVDD ($T\equiv c$) it differs from $f_{\kcm}$ but yields $m_\theta-T_B=r_{\kcm}+g_\theta$ all the same (Appendix~\ref{app:backbone}).

\subsection{Experiments}
\label{sec:kar-exp}

\paragraph{Setup.}
We evaluate \kar{} on the six ADBench datasets where Section~\ref{sec:collapse} found a clear convergence collapse, namely \texttt{annthyroid}, \texttt{Pima}, \texttt{Ionosphere}, \texttt{WDBC}, \texttt{cardio}, and \texttt{vertebral}. Splits are those of Section~\ref{sec:kcm-results}; every number is from our own runs. The collapse metric is the peak-to-final AUROC gap, $\mathrm{gap}=\mathrm{peak}-\mathrm{final}$. The correction $g_\theta$ is a three-layer perceptron trained for $100$ epochs over three seeds, except where a longer horizon is stated, and the anchor radius is swept over $\rho\in\{0,0.5,1,2,5,100\}$. The kernel anchor uses the same bandwidth grid as Section~\ref{sec:kcm} with the untrimmed leave-one-out selector ($\beta=0$).

\paragraph{The anchor radius interpolates KCM and the vanilla detector.}
Table~\ref{tab:kar-rho} sweeps $\rho$. At $\rho=0$ the correction is identically zero and \kar{} reproduces \kcm{} with $\mathrm{gap}=0$ on all six datasets. As $\rho$ grows the bounded correction is free to move: at its best radius \kar{} improves on the kernel-weighted average it anchors to on three of the six datasets, by $+0.019$ on \texttt{annthyroid} and $+0.028$ on \texttt{Pima}, and ties it on two more that are already near saturation; the gains do not extend to the full benchmark (below). The collapse stays bounded; at $\rho=0.5$ the gap is at most $+0.004$. Removing the regularizer, $\rho=100$ in the table, brings the collapse back: the final AUROC falls below \kcm{} on every competitive dataset, by $0.015$ to $0.026$, and the gap widens to as much as $+0.044$, growing with $\rho$ throughout, consistent with the score bound of Proposition~\ref{prop:kar-floor}.

\begin{table}[t]
  \centering
  \caption{\kar{} anchor-radius ablation, six collapse-prone datasets (mean of three seeds, $100$ epochs). $\rho^\star$, the smallest positive radius approaching the best final AUROC, is an \emph{oracle upper bound}. Last columns: regularizer removed ($\rho=100$). $\mathrm{gap}=\mathrm{peak}-\mathrm{final}$.}
  \label{tab:kar-rho}
  \footnotesize
  \setlength{\tabcolsep}{4.5pt}
  \begin{tabular}{@{}lcccccc@{}}
  \toprule
   & & \multicolumn{3}{c}{\kar{} (anchored)} & \multicolumn{2}{c}{reg.\ removed} \\
  \cmidrule(lr){3-5}\cmidrule(lr){6-7}
  \textbf{Dataset} & \textbf{\kcm{}} & \textbf{$\rho^\star$} & \textbf{AUROC} & \textbf{gap} & \textbf{AUROC} & \textbf{gap} \\
  \midrule
  \texttt{annthyroid} & $0.929$ & $0.5$ & $\mathbf{0.948}$ & $+0.000$ & $0.905$ & $+0.044$ \\
  \texttt{Pima}       & $0.701$ & $0.5$ & $\mathbf{0.730}$ & $+0.004$ & $0.676$ & $+0.042$ \\
  \texttt{Ionosphere} & $0.974$ & $1.0$ & $\mathbf{0.976}$ & $+0.001$ & $0.959$ & $+0.003$ \\
  \texttt{WDBC}       & $0.993$ & $1.0$ & $0.993$          & $+0.001$ & $0.976$ & $+0.017$ \\
  \texttt{cardio}     & $0.963$ & $0.5$ & $0.960$          & $+0.002$ & $0.938$ & $+0.014$ \\
  \texttt{vertebral}  & $0.550$ & $0.0$ & $0.550$          & $+0.000$ & $0.562$ & $+0.046$ \\
  \bottomrule
  \end{tabular}
\end{table}

\paragraph{A fixed radius avoids per-dataset label tuning.}
$\rho^\star$ in Table~\ref{tab:kar-rho} is an oracle: it relies on information unavailable at deployment, an upper bound rather than a rule a practitioner can run. The deployable rule is a constant: the radius is expressed in \kcm{}-residual units on $z$-scored features, and $\rho=0.5$ is used unchanged on every dataset in this paper; the $47$-dataset sweep prices the choice. Against the oracle radius, the fixed default costs a mean of $0.0049$ AUROC and a median of $0.0007$ over the $47$ datasets. On five of the six collapse-prone datasets of Table~\ref{tab:kar-rho} it is within $0.001$ of the oracle, and on four it attains it. Eight of the $47$ lose more than $0.005$: the two largest losses ($0.067$, $0.037$) are on datasets where the kernel anchor itself is at or near chance ($0.62$ and $0.48$ AUROC), where \kar{} has nothing to sharpen; the remaining six lose $0.007$ to $0.028$.

\paragraph{The full benchmark.}
We repeat the $\rho$-ablation over all $47$ ADBench datasets, $846$ runs in total. Figure~\ref{fig:kar-adbench} sorts the per-dataset gap. At $\rho=0.5$ \kar{} holds the gap below $0.02$ on $42$ of the $47$ datasets; with the regularizer removed the gap clears $0.02$ on $14$. \kar{} also stays on its anchor: the mean final AUROC at $\rho=0.5$ is within $0.001$ of the $\rho=0$ smoother and within $0.02$ of it on $43$ datasets, while the regularizer-removed limit falls more than $0.01$ below the anchor on $17$. 

\begin{figure}[t]
  \centering
  \includegraphics[width=0.76\columnwidth]{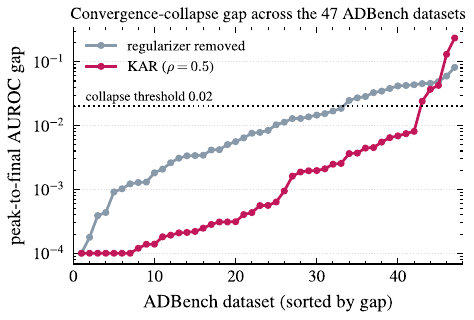}
  \caption{Convergence-collapse gap on the $47$ ADBench datasets, \kar{} at $\rho=0.5$ vs the regularizer removed ($\rho=100$); each series sorted independently, dotted line at the $0.02$ threshold. Mean of three seeds, $846$ runs.}
  \Description{Per-dataset convergence-collapse gap across the 47 ADBench datasets. Two lines, each sorted ascending: KAR at radius rho equals 0.5 stays at or below a 0.02 collapse threshold on 42 of 47 datasets, while the regularizer-removed limit exceeds the threshold on 14 datasets.}
  \label{fig:kar-adbench}
\end{figure}

\paragraph{The three backbones are one detector.}
Instantiated through TCCM, an autoencoder, and Deep SVDD, \kar{} gives the same detector: the three final AUROCs agree to a cross-backbone spread of at most $2.2\times10^{-2}$ on the five competitive datasets, with a median of $2.4\times10^{-3}$, as Proposition~\ref{prop:kar-backbone} predicts.

\paragraph{Prolonged training.}
Figure~\ref{fig:kar-curves} runs the training five times longer, to $500$ epochs, on four of the six datasets. The anchored detector, $\rho=0.5$, holds its AUROC flat across the whole horizon: the gap is at most $+0.0022$ on all four, and the final AUROC stays above both \kcm{} and the unconstrained limit. With the regularizer removed, $\rho=100$, the same backbone drifts down as training continues, and the gap opens to between $+0.018$ and $+0.069$.

\begin{figure*}[t]
  \centering
  \includegraphics[width=\textwidth]{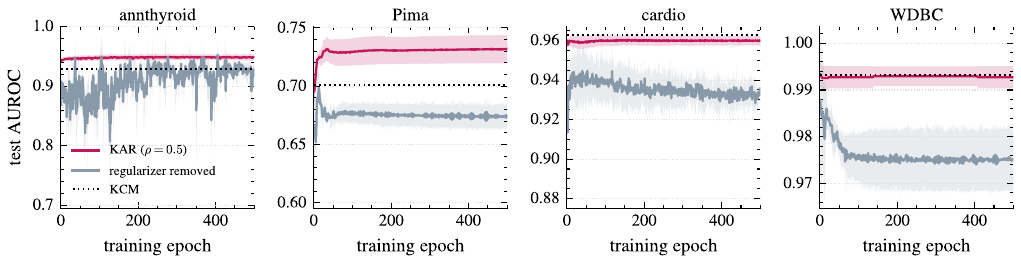}
  \caption{Test AUROC over $500$ epochs (4 of 6 datasets). \kar{} at $\rho=0.5$ (red) stays flat; with the regularizer removed ($\rho=100$, grey) the backbone drifts. Dotted: \kcm{}. Mean of three seeds, shaded $\pm1$ s.d.}
  \Description{Four side-by-side line plots, one per dataset (annthyroid, Pima, cardio, WDBC), showing test AUROC over 500 training epochs. The red curve (KAR with rho equals 0.5) stays high and flat; the grey curve (regularizer removed, rho equals 100) drifts downward as training continues; a black dotted line marks the KCM reference. Shaded bands are one standard deviation over three seeds.}
  \label{fig:kar-curves}
\end{figure*}

\section{Discussion}
\label{sec:discussion}

\kcm{} was built to avoid convergence collapse, and Section~\ref{sec:kcm-results} shows the construction costs nothing in accuracy: a closed-form local smoother is level with the strongest label-free detectors on the benchmark. The mean-AUROC lead over the next methods is a tie within run-to-run noise, so the substance is the asymmetry behind it: the detectors that match \kcm{} are trained neural networks, and \kcm{} matches them with no gradient training, no GPU, and no labels. The accuracy is stable across the benchmark, with mean AUROC about $0.88$ on the low-dimensional datasets ($D\le10$), $0.87$ in the middle range, and $0.85$ on the high-dimensional ones ($D>50$).

The $O(N^2D)$ fit is the main limitation; Section~\ref{sec:kcm-results} (R3) prices it: the $N_0=2000$ cap makes the fit constant in $N$ at a measured cost within $\pm0.01$ AUROC.

\paragraph{Limitations.}
Two limitations bound the \kar{} claim, two more the locality argument itself. \kar{} is anchored to \kcm{}, so where the kernel anchor is near chance it has nothing to sharpen; \texttt{vertebral} is that case, and \kar{} does not help there. Across the full benchmark \kar{} matches the anchor rather than improving on it; the benchmark-wide contribution is the removal of the collapse at no accuracy cost. The locality argument is strongest for \emph{extrapolative} anomalies, points that leave the training data; clustered, conditional, or in-hull anomalies are not separately evaluated here. And necessity is not shown: Proposition~\ref{prop:kar-floor} needs only an off-hull floor on the reference score, so other locality constraints could serve as the anchor; whether they do equally well is untested.

\section{Conclusion}
\label{sec:conclusion}

We identified a convergence collapse in fixed-target anomaly detectors: training the neural map toward its target erases the residual signal off the data, so better optimization makes the detector worse. The cause is a missing locality constraint, and we supplied it twice. \kcm{}, the kernel-regression analog of the fixed-target family, is a closed-form, training-free detector that cannot collapse and matches the strongest trained detectors on ADBench. \kar{} carries the same constraint into a trainable detector, regularizing a neural map to a tube around the kernel-weighted average; it is collapse-proof by construction, holds its accuracy across all $47$ ADBench datasets and three backbones under prolonged training, and improves on the kernel anchor where the data are collapse-prone. Both rest on one principle: an anomaly detector must not be free to extrapolate its on-data behaviour past the data, and anchoring the prediction to a local kernel estimate enforces that, in closed form or as a regularizer.

Accuracy under prolonged training, not at an early peak, is the right metric: the same backbone can swing by $0.1$ AUROC with the stopping epoch, and a paper that quotes only the peak is reporting a regularizer it has not named. And $\kcm{}$ raises the bar a trained detector must clear: where a closed-form kernel smoother matches the best neural detector, the deep methods are paying their training cost for no accuracy gain.

\begin{acks}
This work was performed using HPC resources from the ``M\'e\-so\-cen\-tre'' computing center of
CentraleSup\'elec and \'Ecole Normale Sup\'erieure Paris-Saclay supported by CNRS and R\'egion
\^Ile-de-France (\url{http://mesocentre.centralesupelec.fr/}). This research benefited from the
support of the Chair ``Artificial Intelligence applied to the banking domain and to global health''
led by CentraleSup\'elec and sponsored by Lusis, and by a CIFRE fellowship. This project was
provided with computing AI and storage resources by GENCI at IDRIS thanks to the grant
2026-AD011011828R6 on the supercomputer Jean Zay's H100 partition.
\end{acks}

\appendix

\section{Proofs}
\label{app:proofs}

\begin{proof}[Proof of Theorem~\ref{thm:lazy}]
In the infinite-width limit the tangent kernel is constant through training~\cite{jacotNeuralTangentKernel2020}, and the trained network evolves as a linear model under gradient descent~\cite{leeWideNeuralNetworks2019}. On the training inputs the squared-error objective makes the function-space gradient flow the linear system $\mathrm{d}m_\tau(X)/\mathrm{d}\tau=-\eta\,\Theta(X,X)\,(m_\tau(X)-\mathbf{T})$, whose solution from a zero-output initialization is $m_\tau(X)=(I-e^{-\eta\tau\,\Theta(X,X)})\,\mathbf{T}$. For vector targets the linearized flow is coordinate-wise under the standard NTK parameterization, so the $D$-dimensional system decouples into $D$ scalar copies. The constant kernel propagates this flow to a query as $m_\tau(q)=\Theta(q,X)\,\Theta(X,X)^{-1}m_\tau(X)$, which is Eq.~\eqref{eq:lazy}; sending $\tau\to\infty$ drives the matrix exponential $e^{-\eta\tau\Theta(X,X)}$ to zero.
\end{proof}

\begin{proof}[Proof of Proposition~\ref{prop:kar-equiv}]
$R_\rho$ is the indicator of $\mathcal{C}$, zero on it and $+\infty$ off it, so adding it changes neither the value nor the infima of $L$ on $\mathcal{C}$.
\end{proof}

\begin{proof}[Proof of Proposition~\ref{prop:kar-floor}]
With $T=-x$, $f_{\kcm}(x)=-\hat{x}_h(x)$ by Eq.~\eqref{eq:kcm}, so $m_\theta(x)-T(x)=-\hat{x}_h(x)+g_\theta(x)+x=r_{\kcm}(x)+g_\theta(x)$; taking norms gives the first equality of Eq.~\eqref{eq:kar-floor}. The reverse triangle inequality with $\|g_\theta(x)\|<\rho$ then gives $|\,\|r_{\kcm}(x)+g_\theta(x)\|-\|r_{\kcm}(x)\|\,|<\rho$, the second. Proposition~\ref{prop:immune} bounds $s_{\kcm}(x)\ge\operatorname{dist}(x,\operatorname{conv}(X))$, so $s(x)>s_{\kcm}(x)-\rho\ge\operatorname{dist}(x,\operatorname{conv}(X))-\rho$, unbounded as $x$ leaves the hull, so $s\equiv0$ is excluded for every finite $\rho$.
\end{proof}

\begin{proof}[Proof of Proposition~\ref{prop:kar-backbone}]
Substituting $A_B=T_B+r_{\kcm}$ into $m_\theta=A_B+g_\theta$ gives $m_\theta-T_B=(T_B+r_{\kcm})+g_\theta-T_B=r_{\kcm}+g_\theta$; taking the norm directly gives the second equality of Eq.~\eqref{eq:kar-backbone}, and squaring then taking $\E_x$ gives the first. Both right-hand sides depend only on $g_\theta$, $r_{\kcm}$ and the input distribution, not on $T_B$.
\end{proof}

\section{Backbone independence}
\label{app:backbone}

Table~\ref{tab:kar-backbone} gives the per-dataset breakdown of the cross-backbone result. TCCM and Deep SVDD share the same loss and score in exact arithmetic (Proposition~\ref{prop:kar-backbone}); the AE route is the sign reparameterization $g_\theta\!\to\!-g_\theta$. \texttt{vertebral} is the Scope case of \S\ref{sec:kar-exp}.

\begin{table}[h]
  \centering
  \caption{\kar{} through three fixed-target backbones ($\rho=1$, mean of three seeds, $100$ epochs); \emph{spread} = maximum, over seeds, of the per-seed pairwise gap of the three final AUROCs.}
  \label{tab:kar-backbone}
  \Description{Table with one row per of six datasets and four numeric columns: AUROC via TCCM, via AE, via Deep SVDD, and the cross-backbone spread. On five of the six datasets the three AUROCs differ by less than 2.2e-2; on vertebral the spread reaches 6.2e-2 because the kernel anchor is near random.}
  \scriptsize
  \setlength{\tabcolsep}{5pt}
  \begin{tabular}{@{}lcccc@{}}
  \toprule
  \textbf{Dataset} & \textbf{via TCCM} & \textbf{via AE} & \textbf{via Deep SVDD} & \textbf{spread} \\
  \midrule
  \texttt{annthyroid} & $0.944$ & $0.944$ & $0.945$ & $1.4\!\times\!10^{-3}$ \\
  \texttt{Pima}       & $0.728$ & $0.740$ & $0.728$ & $2.2\!\times\!10^{-2}$ \\
  \texttt{Ionosphere} & $0.976$ & $0.975$ & $0.976$ & $2.4\!\times\!10^{-3}$ \\
  \texttt{WDBC}       & $0.993$ & $0.993$ & $0.993$ & $5.6\!\times\!10^{-4}$ \\
  \texttt{cardio}     & $0.958$ & $0.955$ & $0.958$ & $5.5\!\times\!10^{-3}$ \\
  \texttt{vertebral}  & $0.455$ & $0.472$ & $0.455$ & $6.2\!\times\!10^{-2}$ \\
  \bottomrule
  \end{tabular}
\end{table}

\section{Reproduction}
\label{app:repro}

All experiments use the $47$ ADBench datasets~\cite{hanADBenchAnomalyDetection2022} under the standardized semi-supervised splits ($50\%$ normals for training), with $z$-scored features; the \kcm{} fit and the \kar{} setup are exactly those of Sections~\ref{sec:kcm} and~\ref{sec:kar-exp}, with $u_\theta$ a three-layer width-$256$ ReLU MLP trained with Adam ($10^{-3}$, batch $256$). Full configurations, run logs, and the per-run results behind every table and figure: \url{https://github.com/jose-melo/kernel-contraction-matching}.

\section*{Disclosure of Generative AI Use}
Large language model assistants (Anthropic Claude) were used during the
preparation of this work for grammar checking, writing assistance, and
figure-generation code. All theoretical statements, proofs, experimental
designs and reported results were produced and verified by the authors, who
take full responsibility for the content of this paper. No generative tool is
listed as an author.

\balance
\bibliographystyle{ACM-Reference-Format}
\bibliography{references}

\end{document}